\def\year{2021}\relax
\documentclass[letterpaper]{article} % DO NOT CHANGE THIS
\usepackage{aaai21}  % DO NOT CHANGE THIS
\usepackage{times}  % DO NOT CHANGE THIS
\usepackage{helvet} % DO NOT CHANGE THIS
\usepackage{courier}  % DO NOT CHANGE THIS
\usepackage[hyphens]{url}  % DO NOT CHANGE THIS
\usepackage{graphicx} % DO NOT CHANGE THIS
\usepackage{graphicx}  % DO NOT CHANGE THIS
\usepackage{natbib}  % DO NOT CHANGE THIS AND DO NOT ADD ANY OPTIONS TO IT
\usepackage{caption} % DO NOT CHANGE THIS AND DO NOT ADD ANY OPTIONS TO IT
\usepackage{paralist}
\usepackage[switch]{lineno}

\usepackage{amssymb}
\usepackage{amsmath}

\DeclareMathOperator*{\argmin}{arg\,min}

\usepackage[vlined, ruled, boxed,linesnumbered]{algorithm2e}
\SetKwInOut{Parameters}{Parameters}
\SetKwBlock{Setup}{Setup}{}
\SetKwFor{Loop}{Loop}{}

\usepackage[utf8]{inputenc}
\usepackage[english]{babel}
\usepackage{amsthm}
\theoremstyle{definition}
\newtheorem{definition}{Definition}
\newtheorem*{definition*}{Definition}
\newtheorem{lemma}{Lemma}

\usepackage{multirow}

\newcommand{\LE}{\mathcal{L}}

\begin{document}

% \linenumbers

\title{Generalized Conflict-directed Search for Optimal Ordering Problems}

\author{
    Jingkai Chen,
    Yuening Zhang,
    Cheng Fang,
    Brian C. Williams
    \\
}
\affiliations{
    Massachusetts Institute of Technology \\
    jkchen@csail.mit.edu, zhangyn@mit.edu, cfang@mit.edu, williams@csail.mit.edu
}
\maketitle

\maketitle
\begin{abstract}
\begin{quote}
Solving planning and scheduling problems for multiple tasks with highly coupled state and temporal constraints is notoriously challenging. An appealing approach to effectively decouple the problem is to judiciously order the events such that decisions can be made over sequences of tasks. As many problems encountered in practice are over-constrained, we must instead find relaxed solutions in which certain requirements are dropped. This motivates a formulation of optimality with respect to the costs of relaxing constraints and the problem of finding an optimal ordering under which this relaxing cost is minimum. In this paper, we present Generalized Conflict-directed Ordering (GCDO), a branch-and-bound ordering method that generates an optimal total order of events by leveraging the generalized conflicts of both inconsistency and suboptimality from sub-solvers for cost estimation and solution space pruning. Due to its ability to reason over generalized conflicts, GCDO is much more efficient in finding high-quality total orders than the previous conflict-directed approach CDITO. We demonstrate this by benchmarking on temporal network configuration problems, which involves managing networks over time and makes necessary tradeoffs between network flows against CDITO and Mixed Integer-Linear Programing (MILP). Our algorithm is able to solve two orders of magnitude more benchmark problems to optimality and twice the problems compared to CDITO and MILP within a runtime limit, respectively.
\end{quote}
\end{abstract}

\section{Introduction}

In order to plan for many real-world problems, autonomous systems are required to take into account the requirements over timing and system states. This category of problems ranges from the classical job shop scheduling problems \cite{manne1960job} to hybrid planning problems for multiple tasks with coupled state and temporal constraints \cite{wang2015tburton}. The key to this body of work has been to abstract the tasks, which are then ordered and checked against state and temporal requirements. With a consistent total order, these abstracted tasks are then refined into more concrete courses of actions by resource managers or schedulers. The choice of ordering algorithms is particularly important. A good ordering algorithm should prune unhelpful orderings as much as possible to avoid the computationally expensive checks of the state and temporal consistency. Recent work demonstrates how generalizing inconsistent orderings of events through the interaction with sub-solvers can greatly accelerate this ordering procedure while exploring a special total order tree \cite{wang2015,chen2019efficiently}.

However, the problems specified by these abstract tasks or non-practitioner users are often over-constrained and contain requirements drawing on competing resources. A key challenge to solve such problems is to provide high-quality relaxed solutions in which some requirements are dropped in order to meet hard constraints representing the environment characteristics or other higher-priority requirements. This motivates a notion of optimality with respect to constraint relaxation, as an extension to previous approaches that only considered orderings for constraint satisfaction \cite{wang2015,chen2019efficiently}. While recent work has addressed optimal constraint relaxation problems purely for temporal constraints \cite{yu2013continuously}, we aim to develop an algorithm that can interact with various underlying solvers such that optimal relaxation problems with tightly-coupled state and temporal constraints can be tackled by solving an optimal ordering problem.

In this paper, we introduce Generalized Conflict-directed Ordering (GCDO), an ordering algorithm that generates optimal total orders of the start and end events of abstract tasks. The optimality is defined with respect to a set of constraints with relaxing costs.  GCDO first starts with a total order of these events and then incrementally changes partial orders in a branch-and-bound (B\&B) manner, during which inconsistency or suboptimality discovered by sub-solvers are summarized as bounding constraints to render cost estimation and solution space pruning.

Our optimal ordering algorithm is based on the well-known B\&B search method \cite{lawler1966branch}. The basic idea of B\&B is implicitly enumerating all the solutions and pruning suboptimal subtrees along the way. With different branching and bounding rules, B\&B has been used to solve a wide range of optimization problems. From the perspective of B\&B, as we arrange all the total orders in a special tree \cite{ono2005constant}, we branch by exploring subtrees whose total orders share some common partial orders. Then, we bound and prune the subtrees whose costs are provably suboptimal, given the shared partial orders. 

To estimate the cost of total orders for pruning suboptimal ones, we draw inspiration from the idea of Conflict-directed Incremental Total Ordering (CDITO), which extracts conflicting partial orders from sub-solvers to guide the search \cite{chen2019efficiently}. We further extend these conflicting partial orders to account for costs, and thus the search is able to estimate total order costs without using sub-solvers and directly jump over suboptimal subtrees. The idea of interacting with sub-solvers is also similar in spirit to Satisfiability Module Theories (SMT) solvers such as Z3 \cite{de2008z3} or Optimization Modulo Theories (OMT) solvers such as OptiMathSAT \cite{sebastiani2020optimathsat} while we are using sub-solvers to find relaxations with respect to suboptimality instead of Satisfiability.

% The remainder of this paper is organized as follows. We first introduce a motivating example, an optimal temporal network configuration problem that involves planning for a set of network flows with different priorities and various requirements on duration, loss, delay, and throughput (Section~\ref{section:example}). The necessary tradeoffs between network flows in this example motivate the problem definition of the optimal ordering problems in Section \ref{sec:def}. Then, we present the generalized conflict-directed ordering algorithm in Section \ref{sec:approach}. In Section \ref{sec:results}, we benchmark GCDO against CDITO over the network configuration problems with different complexity and discuss the empirical results. 

\section{Motivating Example}\label{section:example}
Consider a network configuration problem in which we need to schedule and route four network flows with different priorities and allocate bandwidth resources of a network. In this network, the links have different characteristics of loss, delay, and bandwidth capacity, as shown in Figure~\ref{fig:example_topology}.

\begin{figure}[ht]
\centering
\includegraphics[clip,width=0.4\columnwidth]{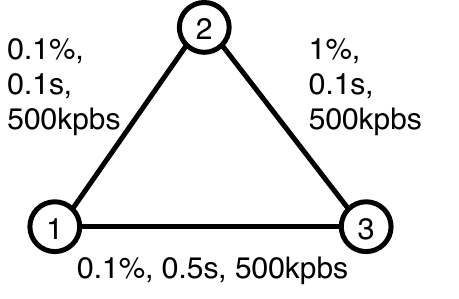}
{\caption{Network topology and statistics.}
\label{fig:example_topology}}
\end{figure}

Table~\ref{tab:example_flow} gives the mission specifications of these four flows on source nodes, destination nodes, maximal loss, maximal delay, and minimal required throughput. A detailed description of these specifications can be found in \cite{chen2018radmax}. The allowable duration lengths for all the flows are [30, 60] seconds. There are also priorities associated with these flows. As Flow-A and Flow-D are very important, they must be transferred. The costs for dropping Flow-B and Flow-C are 3 and 5, respectively. The mission also has temporal requirements: (1) Flow-B and Flow-C should start at the same time, and their ends should be at least 20 seconds apart; (2) Flow-A and Flow-D should start and end at the same times; (3) Flow-B and Flow-C should finish before Flow-A and Flow-D end; (4) we prefer the whole mission to finish in 70 seconds; if the mission takes longer than required, cost 1 is incurred. 

\begin{table}[ht]
\scriptsize
\centering
\begin{tabular}{|c||c|c|c|c|c|c|}
 \hline
 Flow & \ Cost \ & Source & Sink & \  Loss \ & \ Delay \ & Throughput \\
 \hline
 A & $\infty$ & 1 & 2 & 0.5\% & 1s & 200kbps \\  \hline
 B & 5 & 1 & 2 & 3\% & 1s & 360kbps \\  \hline
 C & 3 & 1 & 2 & 3\% & 0.3s & 360kbps \\  \hline
 D & $\infty$ & 1 & 2 & 3\% & 1s & 360kbps \\  \hline
\end{tabular}
{\caption{Mission specifications of flows.}
\label{tab:example_flow}}
\end{table}

A total order of the flow starts and ends with the minimal cost and its corresponding routes are given in Figure~\ref{fig:example_soln}. We also show the temporal constraints of the example in Figure~\ref{fig:example_soln}. It can be easily verified that such a plan only violates the fourth temporal requirement with cost 1. Now we prove this plan is optimal. Given the limited bandwidth constraints, a network link can transfer at most one flow simultaneously. By checking the loss and delay, we know the only feasible route of Flow-A and Flow-C is Path 1-2, while Flow-B and Flow-D can take either Path 1-2 or Path 1-3-2. From the temporal requirements, we know Flow-B and Flow-C must be concurrent, and Flow-A and Flow-D must be concurrent. Furthermore, these two clusters cannot be concurrent, given the limited network capacity, which leads to an 80-second horizon and violates the fourth requirement. Therefore, there is no plan to satisfy all the requirements, and the plan in Figure~\ref{fig:example_soln} that only violates one constraint with cost 1 is optimal.

As we can see, this problem involves managing network flows with different priorities and multiple characteristics, which is hard, especially when multiple flows are considered over a large network \cite{chen2018radmax}. The problem becomes much harder when an exponential number of relaxation choices are considered to minimize total costs. We show that our ordering method can efficiently find an optimal solution for this problem by interacting with proper sub-solvers.

\begin{figure}[t]
\centering
\includegraphics[width=0.9\columnwidth]{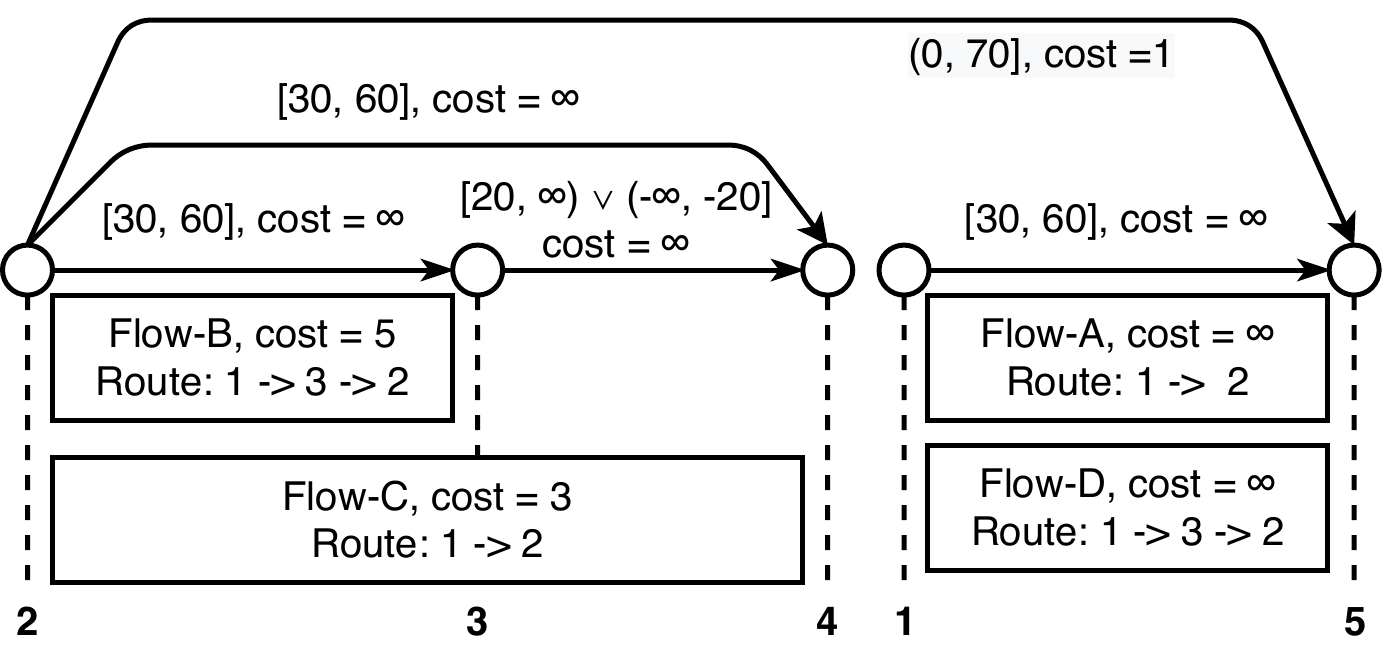}
\caption{Optimal solution of the motivating example.}
\label{fig:example_soln}
\end{figure}

\section{Problem Formulation}
\label{sec:def}
Based on the definitions of general ordering problems \cite{chen2019efficiently}, we associate each constraint with a real-valued cost or an infinite cost, which is similar to the valued constraint satisfaction problems \cite{schiex1995valued}, and the optimal ordering problem is defined as a tuple $\langle E, \Phi, w, h \rangle$:
\begin{itemize}
    \item $E$ is a set of $n$ events represented by the natural numbers $\{1,2,..,n\}$.
    
    \item $\Phi$ is a set of constraints, consisting of a set of ordering constraints $\Phi^+$ and a set of theory constraints $\Phi^*$. An ordering constraint $\phi^+ \in \Phi^+$ is a disjunction of partial orders. A partial order $(a \prec b)$ constrains $a \in E$ to precede $b \in E$. A theory constraint is a user-defined state or temporal requirements across a set of events over time.
    
    \item  $w: \Phi \rightarrow \mathbb{R}_+ \cup \{\infty\}$ is a function that maps a constraint $\phi \in \Phi$ to a positive cost value $g(\phi) $. If $w(\phi) = \infty$, $\phi$ is a hard constraint; otherwise, $\phi$ is a soft constraint.

    \item $h: \LE \times  2^\Phi \rightarrow \{\top, \bot\}$ is a function that maps a total order $\LE$ of $E$ and a set of constraints $\Phi' \subseteq \Phi$ to a Boolean value indicating the consistency of $\LE$ with respect to $\Phi'$.
    
\end{itemize}

 A candidate solution of this ordering problem is a total order $\LE$ that is a sequence of all the events of $E$. Note that we do not allow events to co-occur, and thus we require a strict ordering such that $\lnot (a \prec b) = (b \prec a)$. 

To define the solution consistency and optimality, we first introduce the notions of constraint relaxation and the cost of total orders with respect to its constraint relaxation. A set of constraints $\Phi' \subseteq \Phi$ is a relaxation of $\Phi$ under total order $\LE$ if $h(\LE, \Phi / \Phi') = \top$ (i.e., $\LE$ is consistent with the rest of the constraints). A trivial relaxation under any total order is $\Phi$, which means suspending all the constraints. A relaxation $\Phi^*$ under $\LE$ is an optimal relaxation if $\sum_{\phi \in \Phi^*} w(\phi) \leq \sum_{\phi \in \Phi'} g(\phi)$ holds for any relaxation $\Phi'$ under $\LE$ (i.e., the cost sum of optimal relaxations is minimum). Formally, the cost of $\LE$ is denoted as $g(\LE)$ and defined as 

$${\scriptsize
g(\LE) = \min_{(\Phi' \subseteq \Phi) \land h(\LE, \Phi / \Phi') } \sum_{\phi \in \Phi'} w(\phi)}.$$
Note that, $g(\LE)$ is in the form of $k\infty + c$, where $k$ is a non-negative integer representing the total number of the relaxed hard constraints and $c$ is a real-valued number representing the cost sum of the relaxed soft constraints. We treat $(k\infty+c)$ as finite if $k=0$.

A candidate solution $\LE$ is a solution if and only if its cost $g(\LE)$ is finite. A solution $\LE$ is an optimal solution if and only if $g(\LE) \leq g(\LE')$ holds for any solution $\mathcal{L'}$. 

We assume the cost evaluation function $g$ is provided, which can return the cost of a total order in terms of its optimal relaxation of $\Phi$ given consistency function $h$. 

% In addition, our algorithm needs an available implementation of the bounding constraint extraction function $f$, which that takes as input a total order $\LE$ of $E$ and outputs a set of bounding constraints, which are introduced in Section~\ref{sec:approach:conflict}.

Our motivating example can be formulated as follows:
\begin{itemize}
    \item $E = \{1,2,3,4,5\}$  as shown in Figure~\ref{fig:example_soln}.
    
    \item $\Phi^+ = \{ o_1, o_2, o_3, o_4, o_5 \}$: $o_1 = (1 \prec 5)$, $o_2 = (2 \prec 3)$, and $o_3 = (2 \prec 4)$ constrain each flow's start to precede its end; $o_4 = (3 \prec 5)$ and $o_5 = (4 \prec 5)$ captures temporal requirement (3). All of these constraints are hard constraints with cost $\infty$.
    
    \item $\Phi^* = \{ t_{1}, t_2, t_3, t_4, t_5\} \cup \{ s_1, s_2, s_3, s_4 \}$: $t_1$, $t_2$ and $t_3$ constrain the duration of each flow to be within $[30, 60]$; $t_4$ and $t_5$ captures temporal requirements (1) and (4), respectively; $\{ s_1, s_2, s_3, s_4 \}$ represents the state constraints on routing, loss, delay, and bandwidth of each flow. All of these constraints are hard constraints except $w(t_5) = 1$, $w(s_2) = 5$, and $w(s_3) = 3$.
    
    \item $h$ is able to take as input a total order and determine whether there exists a valid plan,  which specifies flow routing, bandwidth allocation, and schedules, satisfies the given set of constraints and respects this total order. To implement $g$ for this domain, we use the network configuration manager in \cite{chen2018radmax} and the optimal temporal network relaxation solver in \cite{yu2013continuously} to optimize the cost with respect to the state and temporal constraints, respectively. 
    
    % \item $f$ is constructed by applying the method introduced in Section~\ref{sec:approach:conflict} to the solvers  \cite{chen2018radmax} and \cite{yu2013continuously}, which can extract bounding constraints for both state and temporal constraints.
\end{itemize}

One optimal solution of the above problem is $23415$ with cost $1$ by relaxing temporal constraint $t_5$.

\begin{figure*}[ht]
\vspace{-2pt}
\centering
\includegraphics[width=1.99\columnwidth]{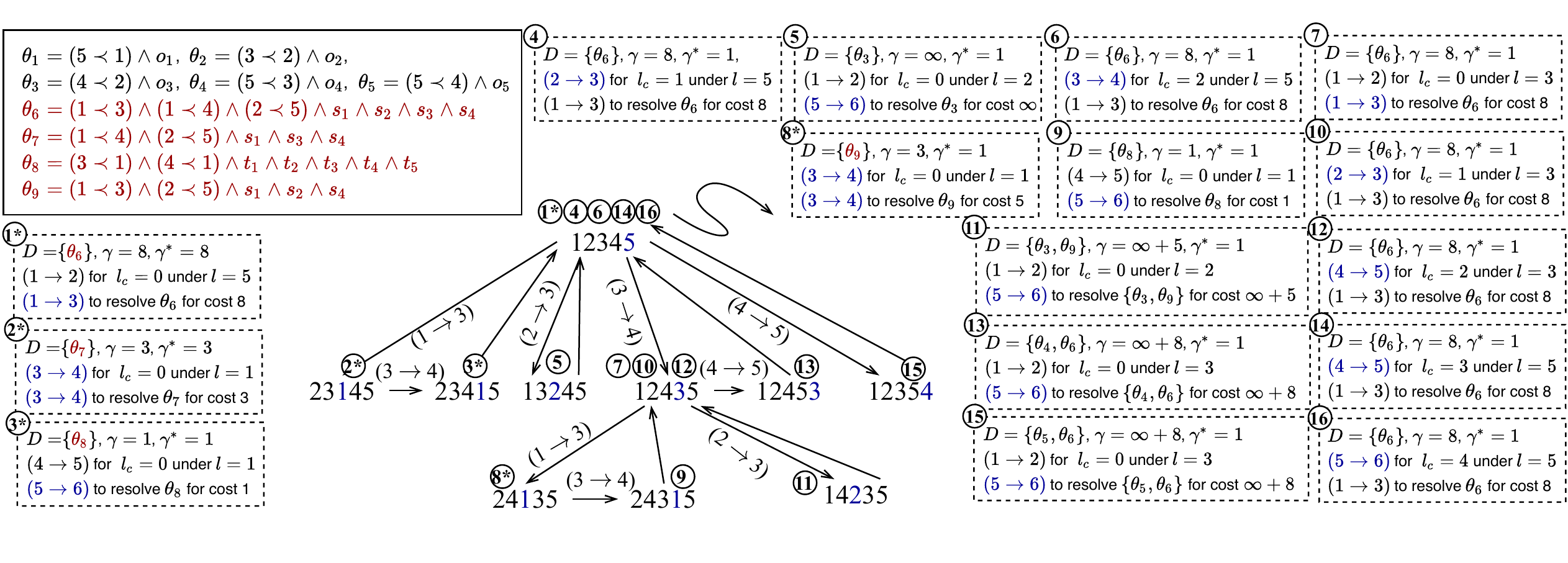}
\vspace{-10pt}
\caption{The explored total orders when solving the motivating example by using GCDO. The levels of total orders  and the chosen next order move are in blue; all the bounding constraints are in the solid-line box; the bounding constraints that are extracted by $f$ during the search are in red; the manifested disjoint bounding constraints $D$, estimation cost $\gamma$, incumbent cost $\gamma^*$, the standard order move, and the first reducing order move at each iteration are given in the dotted-line boxes; we add $*$ to the iteration number if the exact cost is queried from the sub-solvers in that iteration.}
\label{fig:gcdo}
\vspace{-10pt}
\end{figure*}

\section{Approach}
\label{sec:approach}

In this section, we present the design and implementation of the GCDO algorithm, which adopts the well-known branch-and-bound (B\&B) search to systematically explore all the total orders of events and find an optimal ordering solution in an anytime manner. One core idea of B\&B is to bound and prune suboptimal solutions, which requires (1) estimating the objective bounds of subsets of solution candidates; and (2) pruning the suboptimal subsets given these estimations. 

We leverage the following two ideas to pave the way for cost estimation and pruning: 
(1) a well-defined tree structure for enumerating the set of total orders \cite{ono2005constant}, such that total orders in a subtree have partial orders in common;
(2) bounding constraints that summarize the cost of satisfying a set of partial orders and constraints.
With this tree structure and the bounding constraints, we can estimate the cost bound of a total order or all the total orders in a subtree by summing up the costs of their manifested bounding constraints. By using bounding constraints to estimate costs, GCDO avoids many expensive queries about the exact costs of total orders from sub-solvers. Moreover, given the cost estimation of subtrees, we can safely prune inconsistent or suboptimal ones by directly jumping over them.

By using GCDO, the optimal total order of our motivating example can be found in the third iteration as given in Figure~\ref{fig:gcdo}. GCDO starts with total order $12345$, in which all the flows transfer concurrently and leads to total order cost $8$. Then, it identifies that by moving event $1$ after event $3$, the next total order $23145$ resolves this concurrency and reduces the cost by $8$. GCDO then checks $23145$ and finds the concurrency of Flow-A, Flow-C, and Flow-D leads to cost $3$. By moving $1$ after $4$, this concurrency is resolved, and we end up with the optimal solution $23415$ with cost $1$, which only needs to relax the overall makespan constraint $t_5$. These concurrencies are examples of bounding constraints, which are partial orders and constraints that impose a certain cost when satisfied. Then, GCDO takes another thirteen iterations to prove this total order is optimal, during which GCDO mainly uses the bounding constraints to estimate costs and only calls cost evaluation function $g$ once.

GCDO (Algorithm~\ref{alg:gcdo}) takes as input the total number of events $n$, a set of ordering constraints $\Phi^+$, a cost evaluation function $g$, and a bounding constraint extraction function $f$. GCDO outputs either $(\{\}, \infty)$ or an optimal total order $\LE^*$ along with its cost $\gamma^*$ with respect to  $g$ (Line~\ref{line:return}). Total order $\LE$ is initialized as the root total order $(1,2,..,n)$, and the search status $l_c$ is set to $0$ (Line~\ref{line:init1}). Then, Line~\ref{line:init2} initializes the incumbent total order $\LE^*$, the incumbent cost $\gamma^*$, and the extracted bounding constraints $\Theta$. Starting from $(1,2,..,n)$, the algorithm explores all the total orders in a systematic way and updates the incumbents when better solutions are found (Line~\ref{line:updateInc}). We provide a high-level explanation to the pseudo-code in the following four paragraphs and their implementation details are introduced in the rest of Section~\ref{sec:approach}.

\subsubsection{Total Order Search} Our algorithm follows a systematic search strategy in the total order tree introduced by \cite{ono2005constant}. In each iteration, GCDO only maintains one total order $\LE$ and its search status $l_c$,  which is the level of the latest visited children of $\LE$ in the previous iterations. Intuitively, the level of a total order is the first event that is not in the right place compared to the same place as the root total order. For example, $12435$ has level $3$. Line~\ref{line:next} calculates the standard order move to the next total order, which is uniquely determined by $\LE$ and $l_c$. Here an order move is an operation that right shifts the position of an event in a total order. %While the shifted event should not be larger than the current total order level, a smaller event is shifted first. For example, the feasible order moves from $12435$ should only right shift one of the events $\{1,2,3\}$, and right shifting event $1$ is tried first. 
Then, GCDO either moves to the next total order (Lines~\ref{line:next_order_ij}-\ref{line:next_order}) or backtracks (Lines~\ref{line:back_ij}-\ref{line:back}). The function $\textsc{NextMove}$ is implemented by Equation~\ref{eq:search} in Section~\ref{sec:approach:tree} 

\subsubsection{Extracting Bounding Constraints} The algorithm extracts a set of bounding constraints and adds them to $\Theta$ (Line~\ref{line:updateO}) when the true cost of a total order is queried (Line~\ref{line:callg}). Each bounding constraint is a set of partial orders and constraints that impose a certain amount of cost when satisfied.
An example bounding constraint is $\theta_7 = (1 \prec 4) \land (2 \prec 5) \land s_1 \land s_3 \land s_4$ with cost $3$. As partial orders $(1 \prec 4) \land (2 \prec 5)$ force the concurrency of Flow-A, Flow-C, and Flow-D, which cannot be transferred together, any total orders that imply these partial orders will have to at least relax constraint $s3$, which imposes cost $3$. With bounding constraints $\Theta$, we can estimate total order costs without using $g$ (Line~\ref{line:estimate}) and jump further than the standard move to prune inconsistent or suboptimal total orders (Lines \ref{line:reduce}-\ref{line:update_move}). The implementation of function $\textsc{InitCB}$ and the method construct bounding constraint extraction function $f$ are introduced in Section~\ref{sec:approach:extract}.

\subsubsection{Estimating Total Order Costs} At each iteration, we start by calculating an optimistic cost estimation of $\LE$ by using bounding constraints $\Theta$ (Line~\ref{line:estimate}). Our estimation method combines multiple manifested bounding constraints of $\LE$ to determine an informative lower bound estimation while being optimistic. As using $g$ to calculate the true cost of $\LE$ is computationally expensive (Line~\ref{line:callg}), we evaluate $g(\LE)$ only when the estimation is better than the incumbent cost (Line~\ref{line:checkg}). The function $\textsc{EstimateCost}$ is implemented by Equation~\ref{eq:disjoint} in Section~\ref{sec:approach:estimation}.

\subsubsection{Reduction-directed Order Moves} With bounding constraints $\Theta$, the GCDO algorithm computes the first reducing move, which is the first order move that jumps over the inconsistent or suboptimal total orders with respect to the incumbent $\gamma^*$ (Line~\ref{line:reduce}). The order move is then used to update the standard order move to jump further (Line~\ref{line:update_move}). Finding such order moves is based on the observation that some partial orders, which manifest a set of bounding constraints with a total cost larger than the incumbent, can be persistent in some subtrees. Thus, these subtrees and total orders can be pruned without impairing completeness or suboptimality.   The function $\textsc{FirstReducing}$ is implemented by Equations~\ref{eq:cost}-\ref{eq:reduction} in Section~\ref{sec:approach:reduce}.

\begin{algorithm}[t!]\small
\caption{GCDO}
\label{alg:gcdo}
\KwIn{$\langle n, \Phi^+, g, f\rangle$}
\KwOut{$(\LE^*, \gamma^*)$}
$(\LE, l_c) \gets ((1, 2,..,n), 0 )$\label{line:init1}\;
$(\LE^*, \gamma^*, \Theta) \gets (\{\}, \infty, \textsc{InitBC}(\Phi^+)$\label{line:init2}\;
\While{$\LE \ != \ \{\}$}{
    
    $\gamma \gets \textsc{EstimateCost}(\Theta, \LE)$ \label{line:estimate}\;
    \If{$\gamma < \gamma^*$ \label{line:checkg}}
    {$\gamma \gets g(\LE)$ \label{line:callg}\;
    $\Theta \gets \Theta \cup f(\LE)$ \label{line:updateO}\;
    \lIf{$\gamma < \gamma^*$}
    {$(\LE^*, \gamma^*) \gets (\LE, \gamma)$\label{line:updateInc}
    \label{line:updateL}}}

    $(i' \rightarrow j') \gets \textsc{NextMove}(\LE, l_c)$ \label{line:next} \;
    $(i^\dagger \rightarrow j^\dagger) \gets \textsc{FirstReducing}(\mathcal{L}, \Theta, \gamma^*)$ \label{line:reduce}\;

    \lIf{$ni'+j' < ni^\dagger+j^\dagger$}
    {$(i' \rightarrow j') \gets (i^\dagger \rightarrow j^\dagger)$}\label{line:update_move}

    \uIf{$i' < n$}
    {$l_c \gets 0$ \label{line:next_order_ij}\;
    $\LE \gets \LE \oplus (i' \rightarrow j')$ \label{line:next_order}}
    \Else{
    $l_c \gets \textsc{PLV}(\LE)$ \label{line:back_ij}\tcp{\scriptsize position of level event}
    $\LE \gets \textsc{Parent}(\LE)$ \label{line:back}\;}

}
\KwRet{$(\LE^*, \gamma^*)$}\label{line:return}\; 
\end{algorithm}

\subsection{Total Order Search}
\label{sec:approach:tree}

Now we introduce the total order tree \cite{ono2005constant} and a depth-first search strategy in this tree \cite{wang2015}. In the tree of events $E = \{1,2,..,n\}$, nodes are total orders of $E$, and an edge is an operation of altering partial orders of a total order. This tree is rooted at the root total order $(1,2,..n)$ and constructed by expanding all the children of each total order. The tree expansion uses the notions of levels and order moves, which are defined as follows:

\begin{definition}[Level]
\label{def:level}
The level of a total order $\LE = (p_1, p_2,..,p_n) \not = (1,2,..,n)$ is the minimal integer $l$ such that $p_l \not = l$. The level of $(1,2,..,n)$ is $n$.\end{definition}

\begin{definition}[Order Move]
\label{def:move}
An order move $(i \rightarrow j)$ deletes $p_i$ from a total order $\LE = (p_1, p_2,..,p_n)$ and inserts it right after $p_j$ to obtain a total order $\LE'$. This operation is denoted as $\LE' = \LE \oplus (i \rightarrow j)$.
\end{definition}

In this tree, we generate a child by right shifting an event that is less than its parent level. To generate all the children of a total order $\LE$  with level, we apply $\LE \oplus (i \rightarrow j)$ for every $i < l$ and $i < j \leq n$. This tree exactly includes all the total orders of a set of events, which is proved in \cite{ono2005constant}. The total order tree of $E=\{1,2,3,4\}$ is given as an example in Figure~\ref{fig:tree}.

As we can see, the total order level decreases from parents to children. Meanwhile, since the total order tree constrains the feasible order moves with respect to the total orders' level, a portion of partial orders can be persistent in all the total orders in a subtree, which is summarized as Lemma~\ref{lemma:fixed_in_children}:

\begin{lemma}\label{lemma:fixed_in_children}
For a total order with level $l$, order move can only right shift $i < l$ in its subtree, and the partial orders between events $\{l, l+1,..,n\}$ remain in its descendants.
\end{lemma}

Based on this property, as long as we know the partial orders that lead to inconsistency or suboptimality, we can identify the subtrees that always have these partial orders, which can be safely pruned.

% we formally define the total order tree as follows:

% \begin{definition}[Total Order Tree]
% A total order tree rooted at a total order $\LE$ of $E = \{1,2,..,n\}$ with level $l$ is a tree:
% \begin{itemize}
%     \item \textup{\textsc{Nodes}}$(T)$ are a set of total orders of $E$, and each total order is the same as $\LE$ except $\{1,2,..,l-1\}$.
%     \item \textup{\textsc{Edges}}$(T)$ are a set of edges and there is an edge from $\LE'$ to $\LE'$ if and only if $\LE'$ is the parent of $\LE$ . Total order $\LE'$ with level $l'$ is the  parent of $\LE''$ with level $l''$ if and only if $l' > l''$ and $\LE'$ and $\LE''$ are identical except $l''$.  
% \end{itemize} 
% \end{definition}

% \begin{lemma}\label{lemma:fixed_in_children}
% For a total order with level $l$, the partial orders between the events $\{l, l+1,..,n\}$ remain in its descendants.
% \end{lemma}

To search this tree, we follow the depth-first order: (1) from a total order, the algorithm first visits its children and then its siblings with the same level; (2) when its children and these siblings are exhausted, the algorithm backtracks to its parent; (3) The group of children with the lowest level are generated first, and within each group, children are generated in the order of right shifting children's level events until the right end; (4) the same-level siblings are also generated by right shifting their level events until the right end. The order of visiting all the total orders of four events is given in Figure~\ref{fig:tree}. Formally, from a total order $\LE = (p_1, p_2,..,p_n)$, when its latest visited child has level $l_c$, the next move $
\textsc{NextMove}(i,j,\LE)$ is calculated as follows:

\begin{equation}\footnotesize
\label{eq:search}
\begin{cases} 
      (l_c + 1 \rightarrow l_c + 2) & (l_c < l -1) \\
      (\textsc{PLV}(\LE) \rightarrow  \textsc{PLV}(\LE)+ 1) & (l_c = l -1)
 \end{cases},
\end{equation}
where $\textsc{PLV}(\LE)$ is the position of the level of $\LE$ in itself. Note that the feasible order moves under these two conditions lead to the children of $\LE$ and its siblings with level $l$, respectively. By following Equation~\ref{eq:search}, $(i \rightarrow j)$ with lower $(ni+j)$ is taken first from a total order and $(ni+j) \leq (nl+n)$ holds for the feasible order moves. Feasible order moves are simply applied as Lines~\ref{line:next_order_ij}-\ref{line:next_order}. When the returned move is $(n \rightarrow n+1)$, which is infeasible and means all the children and same-level siblings are exhausted, the algorithm backtracks to the parent of $\LE$ (Lines~\ref{line:back_ij}-\ref{line:back}).

\begin{figure}[t!]
\centering
\includegraphics[width=0.9\columnwidth]{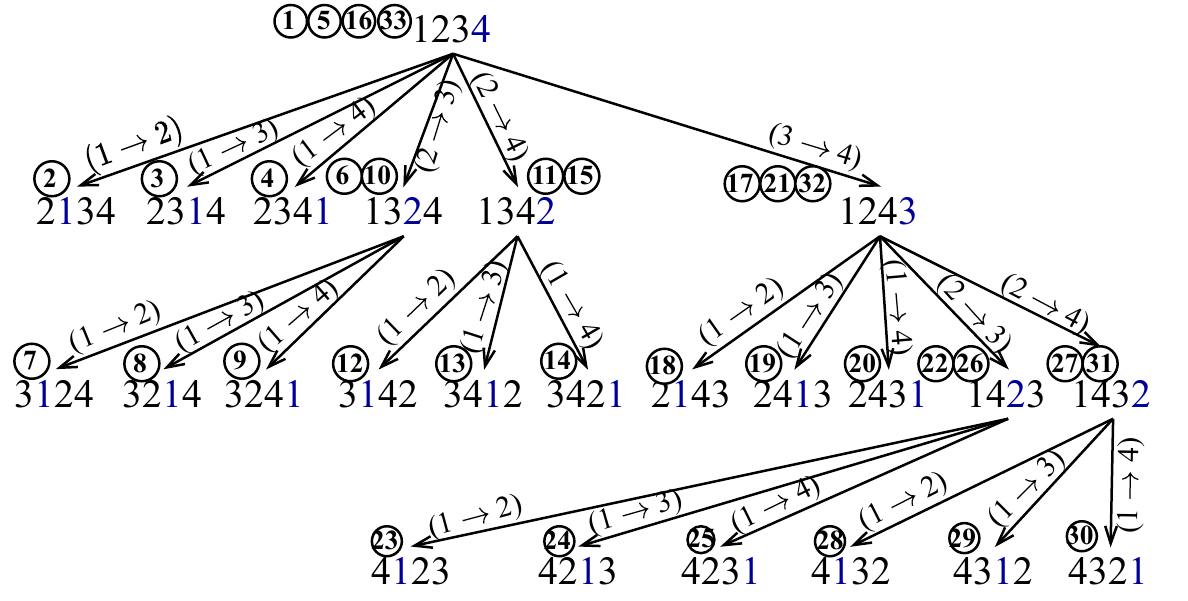}
\caption{Total order tree of $E = \{1,2,3,4\}$. The levels of total orders are in blue. The orderings of visiting these total orders by following Equation~\ref{eq:search} are given in their upper left.}
\label{fig:tree}
\vspace{-10pt}
\end{figure}

\subsection{Extracting Bounding Constraints}
\label{sec:approach:extract}
In this section, we introduce the formal definition of bounding constraints and their extraction method. Intuitively, the bounding constraints are a set of partial orders along with the constraint that imposes a certain amount of cost when satisfied, which are the generalization of ordering conflicts to account for costs \cite{chen2019efficiently}. They are also similar in spirit to valued nogoods \cite{dago1996nogood} and bounding conflicts  \cite{timmons2020best} by associating partial assignments with objective value bounds, but the bounding constraints are tailored to the ordering problems by replacing partial assignments with partial orders.

In our motivating example, Flow-A, Flow-C, and Flow-D cannot be transferred concurrently because of the limited bandwidth capacity of the two available paths. When a total order forces them to be concurrent, dropping Flow-C is the optimal relaxation since the other flows have higher priorities and thus higher costs to drop. As Flow-C starts at $2$ and ends at $4$, and Flow-A and Flow-D start at $1$ and end at $5$, this concurrency can be summarized as the partial orders $(1 \prec 4) \land (2 \prec 5)$, which compactly capture all the combinations of this concurrency: $1245$, $1254$, $2145$, and $2154$. As the state constraints of these flows are $\{s_1, s_2, s_3\}$, the fact that this concurrency imposes at least cost $3$ can be summarized as a bounding constraint $\theta_7 = (1 \prec 4) \land (2 \prec 5) \land s_1 \land s_3 \land s_4$ with cost $3$. Another bounding constraint example for state constraints is $\theta_6 = \textsc{PO}(\theta_6) = (1 \prec 3) \land (1 \prec 4) \land (2 \prec 5) \land s_1 \land s_2 \land s_3 \land s_4$ with cost $1$, which summarizes that the concurrency of all the flows has cost $8$ by relaxing Flow-B and Flow-C.

In addition to the bounding constraints for state constraints, a bounding constraint example of temporal constraints is $\theta_8 = (3 \prec 1) \land (4 \prec 1) \land t_1 \land t_2 \land \land t_3 \land t_4 \land t_5$ with cost $1$, which is manifested by total order $23415$ as in Figure~\ref{fig:example_soln}. As the partial orders $(3 \prec 1) \land  (4 \prec 1)$ force both Flow-B and Flow-C to end before Flow-A and Flow-D start, the mission takes at least $80$ seconds given temporal constraints $\{t_1, t_2, t_3, t_4\}$, which contradicts with time limit $70$ seconds specified by $t_5$. As relaxing $t_5$ will remove this time limit with cost $1$, which is the optimal relaxation since the other involved constraints are hard constraints, this bounding constraint has cost $1$.

The third kind of bounding constraints is for ordering constraints, which is simpler than those for state and temporal constraints. For example, total order $12453$ includes partial order $(5 \prec 3)$ and violates the ordering constraint $o_4 = (3 \prec 5)$ with cost $\infty$, which can be summarized as a bounding constraint $\theta_4$ with cost $\infty$. Formally, we define the bounding constraints as follows:

\begin{definition}[Bounding Constraints]\label{def:bounding_ordering} 
Let $\Phi$ be a set of constraints. A bounding constraint $\theta$ associated with cost $\textsc{cost}(\theta)$ is a conjunction of a set of partial orders $\textsc{PO}(\theta)$ and a set of constraints $\textsc{CS}(\theta) \subseteq \Phi$ such that $\textsc{cost}(\theta)$ is the cost sum of the optimal relaxation of $\textsc{CS}(\theta)$ under $\textsc{PO}(\theta)$. We say a total order $\mathcal{L}$ manifests $\theta$ if $\mathcal{L}$ implies $\textsc{PO}(\theta)$.
\end{definition}

For ordering constraint $\phi^+_r = \lor q_{rs}$, where $q_{rs}$ is a partial order, we can obtain its corresponding bounding constraint $\theta_r$ by having $\textsc{PO}(\theta_r) = \land q_{rs}$, $\textsc{CS}(\theta_r) = \{\phi^+_r\}$, and  $\textsc{cost}(\theta_r) = w(o_r)$. As this procedure is easy, we extract the bounding constraints for all the ordering constraints $\Phi^+$ before search, which is implemented as function \textsc{InitBC}.

Now we introduce our method to construct bounding extraction function $f$, which finds $\textsc{PO}$, $\textsc{CS}$, and $\textsc{cost}$ of bounding constraints for state and temporal constraints. First, we extract $\textsc{PO}$ for state and temporal constraints by using the same approach to extract ordering conflicts as \cite{chen2019efficiently}: we represent the concurrency of state constraints as a polynomial number of partial orders; and we obtain the partial orders that are inconsistent with a set of temporal constraints by collecting the imposed partial orders in their negative cycle \cite{dechter1991temporal}. They are given in Equation~\ref{eq:inconsistent_state} and Equation~\ref{eq:inconsistent_time}, respectively.

The partial orders $\textsc{PO}^s$ to present the concurrency of multiple state constraints are:

\begin{equation}\small
\label{eq:inconsistent_state}
    \textsc{PO}^{s} = \underset{i,j}{\land}R^{s}_{ij} = \underset{i,j}{\land} (x_i^\vdash \prec x_j^\dashv).
\end{equation}
where each $R^s_{ij} = (x_i^\vdash \prec x_j^\dashv) \land (x_j^\vdash \prec x_i^\dashv)$ represents the concurrency of two tasks (i.e., state constraints), and $x_i^\vdash$ and $x_i^\dashv$ are the start and end events of the $i^{\text{th}}$ task.

When a negative cycle in a simple temporal network is found, the corresponding partial orders $\textsc{PO}^t$ are:

\begin{equation}\small
\label{eq:inconsistent_time}
    \textsc{PO}^{t} = \underset{i}{\land}R^t_i = \underset{i}{\land}(x_i^- \prec x_i^+),
\end{equation}
where $R^t_i = (x_i^- \prec x_i^+)$ represents the from event and to event of a temporal constraint that is involved in the cycle and added because of total ordering.

Then, $\textsc{CS}$ in these two cases are the concurrent state constraints and all the temporal constraints involved in the negative cycle, respectively.  Lastly, $\textsc{cost}$ can be obtained by associating the state and temporal constraints with costs and solving an optimization problem for an optimal relaxation. In our example, we use the solvers in \cite{chen2018radmax} and \cite{yu2013continuously} to find the optimal relaxation with respect to the state and temporal constraints, respectively, and $\textsc{cost}$ is the cost sum of the relaxed constraints.

\subsection{Estimating Total Order Costs}
\label{sec:approach:estimation}
Given a set of bounding constraints $\Theta$ manifested by a total order $\LE$, we can calculate $\gamma^\Theta$, an optimistic, informative cost estimation of $\LE$ by using $\Theta$ instead of evaluating the exact cost $g(\LE)$. The latter usually requires solving complex optimization problems, which is more computationally expensive and should be avoided as much as possible. For example, without using $g$, we know the costs of total orders that manifest $\theta_6$ with $\textsc{PO}(\theta_6) = (1 \prec 3) \land (1 \prec 4) \land (2 \prec 5)$ such as $12435$ and $12453$ are at least $8$.

An informative cost estimation of a total order should be as large as possible to lower bound its true cost. To be informative, we want to incorporate the information of multiple bounding constraints. For example, total order $12453$ manifests both bounding constraints $\theta_5$ with cost $\infty$ and $\theta_6$ with cost $8$. Therefore, the cost estimation of $12435$ is lower bounded by $(\infty+8)$ by summing their costs together. 

Meanwhile, we still require the estimation to be optimistic when considering multiple bounding constraints. An optimistic cost estimation is a lower bound of the true cost. Thus, the constraint sets of the used bounding constraints should neither count the costs of unnecessary relaxations nor double count same relaxations.  For example, total order $12435$ manifests both bounding constraints $\theta_7$ with cost $3$ and $\theta_6$ with cost $8$. An optimistic cost estimation of $12435$ is $8$, which is determined by $\theta_6$ instead of counting both $\theta_7$ and $\theta_6$. This is because they share $s_3$ (i.e., Flow-C) in their optimal relaxation, and its cost should be counted only once. We choose $\theta_6$ with a higher cost to get a closer bound to its true cost.

\subsubsection{Disjoint Bounding Constraints} Formally, when we use the cost sum of multiple bounding constraints to obtain an informative estimation, the key to being optimistic is to use a set of disjoint bounding constraints, which is formally defined as follows:

\begin{definition}[Disjoint Bounding Constraints]\label{def:disjoint}
Two bounding constraints $\theta_i$ and $\theta_j$ are disjoint if the intersection of their constraint sets $\textsc{CS}(\theta_i) \cap \textsc{CS}(\theta_j)$ only include hard constraints.
\end{definition}

Given a set of disjoint bounding constraints $D$ manifested by total order $\LE$, we can estimate the cost of $\LE$ as $\sum_{\theta_r \in D} \textsc{cost}(\theta_r)$. Based on Definition~\ref{def:disjoint}, we conclude the optimism of this estimation in Lemma~\ref{lemma:optimistic}.

\begin{lemma}\label{lemma:optimistic}
Let $\LE$ be a total order with cost $\gamma$ and $D$ be a set of disjoint bounding constraints manifested by $\LE$. Let $\sum_{\theta_r \in D} \textsc{cost}(\theta_r) = k\infty+c$ be a cost estimation of $\LE$. If $k = 0$, we have $c \leq \gamma$; otherwise, $\gamma \geq \infty$.
\end{lemma}

\begin{proof}

When $k = 0$ for $k\infty + c$, the estimation is optimistic because: given any two disjoint bounding constraints $\theta_i$ and $\theta_j$, we know the intersection of their constraint sets $\textsc{CS}(\theta_i) \cap \textsc{CS}(\theta_j)$ have only hard constraints, and thus soft constraints can not appear in the intersection of their relaxations. Therefore, the choice of the optimal relaxation for each bounding constraint over a set of soft constraints is independent and thus remains optimal. When $k > 0$ for $k\infty + c$, there must be a bounding constraint with cost $\infty$, which is manifested by total order $\LE$, and thus $\LE$ must be inconsistent and has at lest cost $\infty$.
\end{proof} 

\subsubsection{Finding Informative Optimistic Estimation}
Given a set of bounding constraints $\Theta$ manifested by total order $\mathcal{L}$, we calculate an informative estimation $\gamma^{\Theta}$  by choosing a set of disjoint bounding constraints $D$ with the maximal cost sum:

\begin{equation}\label{eq:disjoint}
    \gamma^{\Theta} = \max_{(D \subseteq \Theta) \land (D \text{ is disjoint})} \sum_{\theta_r \in D} \textsc{cost}(\theta_r)
\end{equation}

As an implementation of $\textsc{EstimateCost}$, we find such disjoint bounding orderings $D \subseteq \Theta$ and its cost $\gamma^\Theta$ by solving a weighted maximum clique problem \cite{bomze1999maximum}, which is NP-hard but can be solved very efficiently in practice \cite{cai2016fast,balas1996weighted}. We first construct an undirected  graph: the vertices are bounding constraint $\Theta$, the weight of each vertex $\theta \in \Theta$ is $\textsc{cost}(\theta)$, and there is an undirected edge between two vertices $\theta_i, \theta_j \in \Theta$ if and only if they are disjoint as given in Definition~\ref{def:disjoint}.  A clique $D$ is a subset of vertices such that every two distinct vertices in this clique are adjacent, and a maximum clique is not a subset of any other clique.  
In the graph, each maximum clique $D$ represents a set of disjoint bounding constraints. To be informative, we choose the clique with the maximum cost sum to have a tight bound. 

% When $k > 0$, this total order is inconsistent, and we may consider the same violation of hard constraints many times. This does not affect the optimism of our method since a solution needs to satisfy all the hard constraints anyway. As resolving more bounding constraints for inconsistency can potentially help prune more total orders, we prefer a larger $k$ as well. 

\subsection{Reduction-directed Order Moves} 
\label{sec:approach:reduce}
In this section, we introduce the method to find the next order move that jumps over inconsistent or suboptimal total orders given the incumbent cost. We first introduce our method to find the first resolving move of a bounding constraint $\theta$, before which any total order still manifests $\theta$ and thus has at least cost $\textsc{cost}(\theta)$.  By using these first resolving moves, we identify the first reducing move of a desired cost reduction, before which any total order still manifests a set of bounding constraints and is inconsistent or suboptimal.

\subsubsection{First Resolving Move}
\label{sec:approach:cdito:resolution}

To find the first resolving move of a bounding constraint, we use the same method as resolving ordering conflicts in \cite{chen2019efficiently}, which is based on Lemma~\ref{lemma:fixed_in_children} and finds the first order move that leads to a total order negating at least one partial order in its partial orders.

Consider total order $12453$ that manifests bounding constraints $\{\theta_4, \theta_6\}$ in the thirteen iteration. We consider finding the first resolving move of one of its bounding constraints, $\theta_6$, as an example. From $12453$, the standard next move calculated by Equation~\ref{eq:search} is $(1 \rightarrow 2)$ and leads to $21435$, which still implies $\textsc{PO}(\theta_6)=(1 \prec 3) \land (1 \prec 4) \land (2 \prec 5)$. To resolve $\theta_6$, we should jump over $21453$ and directly move to $24153$ through order move $(1 \rightarrow 3)$, which negates $(1 \prec 4)$ in $\textsc{PO}(\theta_6)$ and can reduce $\textsc{cost}(\theta_6)$. Thus, $(1 \rightarrow 3)$ is the first resolving move of $\theta$ from $12453$. While any order move before $(1 \rightarrow 3)$ is nonhelpful, the order moves after it may also resolve $\theta_6$ . For example, $(1 \rightarrow 4)$ moves to $24513$, which negates both $(1 \prec 3)$ and $(1 \prec 4)$. Another example is $(2 \rightarrow 3)$. Even though the new total order $14253$ still implies $\textsc{PO}(\theta_6)$, its subtree may contain total orders that resolve $\theta_6$ such as taking $(1 \rightarrow 2)$ in subsequent.

We return infeasible moves $(n \rightarrow n+1)$ when there is no total order that resolves the bounding constraint among the descendants of the current total order, its same-level siblings, and the descendants of these siblings. Consider the other bounding constraint $\theta_4$ manifested by $12453$. As the level of $12453$ is $3$, the feasible order moves from $12453$ can only right shift events $1$ or $2$ to generate its children or right shift $3$ to generate its siblings. However, $3$ has reached the right end, and right shifting $1$ or $2$ does not change partial order $\textsc{PO}(\theta_4) = (5 \prec 3)$. Moreover, the levels of the new total orders obtained by shifting $1$ or $2$ are no more than $2$, and thus there is no chance to negate $(5 \prec 3)$ in the following moves by Lemma~\ref{lemma:fixed_in_children}. Thus, the first resolving move of $\theta_4$ is $(5 \rightarrow 6)$, which means backtrack.

Formally, consider a bounding constraint $\theta_r$ with partial orders $\textsc{PO}(\theta_r) = \underset{s}{\land} q_{rs}$ manifested by total order $\LE=(p_1,p_2,...,p_n)$ with level $l$. The first resolving order move of $\theta_r$ from $\LE$ is

\begin{equation}
\label{eq:single_resolution}\small
    (i^{\dagger}_r \rightarrow j^{\dagger}_r) = \underset{(i' \rightarrow j') \in \Pi_r}{\text{argmin}} (ni'+j'),
\end{equation}
where $\Pi_r= \{(i' \rightarrow j') \, | \, ((p_{i'} \prec p_{j'}) \in \textsc{PO}(\theta_r))  \land (p_{i'} \leq l)\} \cup \{(n \rightarrow n+1)\}$. Order move $(n, n+1)$ is infeasible and means GCDO should backtrack as in Equation~\ref{eq:search}.

% We then prove that any order move before the order move generated by Equation~\ref{eq:single_resolution} cannot resolve $\textsc{PO}(\theta_r)$: (1) We first consider the case $j<i$, which implies the ancestors of $\LE$ have right shifted $j$ after $i$ and thus the level $l$ of $\LE$ is not larger than $j$. As $l \leq j < i$, we know $(i \prec j)$ remains in its descendants by Lemma~\ref{lemma:fixed_in_children}. Thus, no total order can negate this partial order in the subtree of $\LE$. (2) We then consider the case $j \geq i$. For any feasible order move $(i' \rightarrow j')$ with $(ni'+j')<(ni+j)$ and thus $i' \leq i$, as the new total order $\LE' = \LE \oplus (i' \rightarrow j')$ has level $l' = i' \leq i < j$, partial order $(i \rightarrow j)$ remains in the descendants of $\LE'$ by Lemma~\ref{lemma:fixed_in_children}. Thus, no total order can negate this partial order before taking $(i \rightarrow j)$.
\begin{table*}[ht]\small
\centering
\begin{tabular}{|c||c|c|c|c|c||c|c|c|c|c||c|c|c|c|}
\hline
\multirow{2}{*}{\textbf{\#flows}} & \multicolumn{5}{l|}{\hspace{0.23\columnwidth}\textbf{GCDO}} & \multicolumn{5}{l|}{\hspace{0.23\columnwidth}\textbf{CDITO}} & \multicolumn{4}{l|}{\hspace{0.23\columnwidth}\textbf{MILP}} \\ 
\cline{2-15} 
& \ \ \ $t_1$ \ \ \  & \ \ \ $\gamma_1$ \ \ \ & \ \ \  $\gamma$ \ \ \  & \ $\eta$ \ & \ \ \  $\zeta$ \ \ \  & \ \ \  $t_1$ \ \ \  & \ \ \  $\gamma_1$ \ \ \   & \  \ \  $\gamma$ \ \ \   & \ $\eta$  \ & \ \ \  $\zeta$ \ \ \ & \ \ \  $t_1$ \ \ \  & \ \ \  $\gamma_1$ \ \ \   & \  \ \  $\gamma$ \ \ \   & \  $\eta$  \    \\ \hline
5 & 0.029 & 2.30 & \textbf{0} & \textbf{100} & 0.087 & \textbf{0.020} & 2.30 & \textbf{0} & \textbf{100} & 0.763 & 0.062 & 4 & \textbf{0} & \textbf{100} \\ \hline
10 & 0.043 & 2.87 & \textbf{0.43} & \textbf{81} & 0.086 & \textbf{0.036} & 2.87 & 1.52 & 9 & 0.751 & 0.273 & 8 & 0.73 & 74 \\ \hline
15 & 0.096 & 3.96 & \textbf{0.41} & \textbf{80} & 0.043 & \textbf{0.089} & 3.96 & 3.04 & 3 & 0.727 & 1.055 & 12 & 1.14 & 41 \\ \hline
20 & 0.172 & 5.18 & \textbf{0.47} & \textbf{77} & 0.008 & \textbf{0.153} & 5.18 & 4.27 & 0 & 0.706 & 1.791 & 16 & 1.09 & 52  \\ \hline
25 & 0.301 & 6.12 & \textbf{0.49} & \textbf{73} & 0.006 & \textbf{0.229} & 6.12 & 5.24 & 0 & 0.681 & 3.516 & 20 & 1.33 & 47 \\  \hline
30 & 0.460 & 7.24 & \textbf{1.53} & \textbf{57} & 0.004  & \textbf{0.317} & 7.24 & 6.54 & 0 & 0.676 & 4.949 & 24 & 1.95 & 19 \\ \hline
\end{tabular}
\caption{Experimental results. $t_1$: the average runtime to find the first solution; $\gamma_1$: the average cost of the first solution; $\gamma$: the average cost of the returned solutions. $\eta$: the number of the problems whose returned solutions are optimal. $\zeta$: the ratio of the average times of calls to $g$  to the average number of the explored total orders. We highlight the best results of $t_1$, $\gamma$, and $\eta$.}
% cfang: for \zeta, maybe: ratio of average of evaluations to the average number of explored total orders. Also, instead of that ratio, maybe it would be better to present average(log(evaluations/total orders))
\label{table:result}
\vspace{-10pt}
\end{table*}

\subsubsection{First Reducing Move}
\label{sec:approach:cdito:resolution}
Let $\Delta = \gamma^D - \gamma^*$ be the gap between an optimistic cost estimation $\gamma^D$ given disjoint bounding constraints $D$ and the incumbent cost $\gamma^*$. We need to find the first order move that resolves a set of bounding constraints $D' \subseteq D$ with enough cost reduction $\sum_{\theta \in D'}  \textsc{cost}(\theta)$ such that $\sum_{\theta \in D'}  \textsc{cost}(\theta) > \Delta$. The first order move to achieve this reduction $\Delta$ is called the first reducing move of $\Delta$ from $\mathcal{L}$ given $D$. From another perspective, any order move that is before this move must lead to the total orders whose costs are at least $\gamma^*$ and thus not worth exploring.

In our previous example, we consider total order $12453$ with manifested disjoint bounding constraints $\{\theta_4, \theta_6\}$ and incumbent cost $1$ in the thirteen iteration. The first resolving move of $\theta_4$ and $\theta_6$ are $(1 \rightarrow 3)$ and $(5 \rightarrow 6)$, respectively. Recall that any order move before the first resolving move of a bounding constraint will lead to a subtree where this bounding constraint remains. Thus, as we identify reducing cost from $\infty + 8$ to be less than incumbent $\gamma^*=1$ requires resolving both constraints, we choose the order move that jumps furthest, which is $(5 \rightarrow 6)$ in this example. Thus, the search can safely backtrack by considering these two bounding constraints together.

To identify the first reducing move of incumbent cost $\gamma^*$ from $\LE$ as an implementation of function $\textsc{FirstReducing}$, we start by calculating the first resolving move $(i^\dagger_r \rightarrow i^\dagger_r)$ of each bounding constraint $\theta_r \in D$. Then, we sort all these resolving moves in the order of increasing $(ni^\dagger_r + i^\dagger_r)$, which is the exploration order of the order moves as given in Equation~\ref{eq:search}. We also associate every move with a cost estimation $G^D_r$, which is as follows:

\begin{equation}\label{eq:cost}\small
    G^D_r  = \sum_{k \in \{r+1,r+2,..R\}} \textsc{cost}(\theta_k),
\end{equation}
where $R$ is the number of all the bounding constraints $D$. This cost estimation $G^D_r$ is the cost sum of all the unresolved bounding constraints of $(i^\dagger_r \rightarrow i^\dagger_r)$ from the current total order, which is optimistic following Lemma~\ref{lemma:optimistic}. Therefore, if an order move is before $(i^\dagger_r \rightarrow j^\dagger_r)$ and $G^D_{r-1} > \gamma^*$, it must lead to a total order or a subtree that is inconsistent or suboptimal, which can be safely pruned. Given the current incumbent cost $\gamma^*$, we identify the next resolving move to reduce the cost to be under $\gamma^*$ as follows:

\begin{equation}\label{eq:reduction}\small
    (i^\dagger \rightarrow j^\dagger) = \argmin_{(i^\dagger_r \rightarrow i^\dagger_r) \land G^D_r < \gamma^*} (ni^\dagger_r+j^\dagger_r). 
\end{equation}

\section{Experimental Results}
\label{sec:results}
In order to evaluate the effects of using bounding constraints in GCDO, we benchmarked GCDO on the optimal temporal network configuration problems similar to our motivating example, 
with different complexities and sizes. These problems involve routing flows and allocating bandwidth resources with respect to requirements on loss, delay, bandwidth against CDITO \cite{chen2019efficiently} and Mixed Linear Inter Programming (MILP) approach by using Gurobi \cite{gurobi2020gurobi}. CDITO can interact with a rich class of sub-solvers as our algorithm does. As CDITO is only aware of hard constraints to resolve ordering conflicts, we modified CDITO to keep exploring the solution space after finding the first solution and record the best solution as the incumbent. For both CDITO and GCDO, we use \cite{chen2019efficiently} and \cite{yu2013continuously} as sub-solvers for state and temporal constraints, respectively. The MILP encoding uses Boolean variables to indicate the precedence relations of events, and constraints and violating costs are conditional on these variables.

We use the same communication network simulator in \cite{chen2018radmax} to generates network flow requirements on a meshed network. The major difference is that we associate each flow with a priority to be its relaxing cost. While one-fifth of the flows must be transferred, the relaxing costs of the others are $1$. The other setup is as follows: (1) the mission horizon is 300s; (2) the meshed network has 6 nodes; (3) the loss, delay, and bandwidth of each link are uniformly generated from intervals [0.1,0.3]\%, [0.1,0.3]s, and [500,1000]kbps; (4) the loss, delay, throughput, minimum duration of each network flow are uniformly generated from  [0.1,0.3]\%, [0.1,0.3]s, [600,1000]kbps, and [20,80]s; (5) the generator adds temporal constraints between randomly chosen events with a duration (0,100], and the number of temporal constraints is one-fifth of the number of flows; 

We tested six scenarios of 5, 10, 15, 20, 25, and 30 flows with 100 trials each. The timeout was 30 seconds. We only include the trials in which consistent solutions exist.

Table~\ref{table:result} shows the experimental results. We observe that all GCDO and CDITO can find consistent solutions ten times faster than MILP. As GCDO reduces to CDITO before an incumbent solution is found except recording bounding constraints, their first solutions are exactly the same, and GCDO spends slightly more time on recording these constraints.  In the first solutions, MILP basically drops all the optional flows. It can be seen that while GCDO then reduces at least $80\%$ costs for most of the scenarios, CDITO can only achieve good final cost $\gamma$ when $\#\text{flow} \leq 10$ and fails to achieving more than $20\%$ reduction for the other scenarios in 30s. This demonstrates that GCDO is capable of using suboptimality to efficiently guide the search for high-quality solutions. Meanwhile, the costs of the solutions returned by GCDO is at lest half of that of MILP for most cases.
We also report $\eta$ , the number of problems in which an optimal solution is proved, which requires the algorithms to exhaust the solution space. As GCDO is able to prove the optimality of the returned solutions for a large portion of problems in $30$ seconds, CDITO fails to complete this task for most problems and MILP only performs well when $\#\text{flow} \leq 10$. The calls to the sub-solvers dominate the runtime for both CDITO and GCDO given our observation that more than 95\% runtime is spent on these calls. The major reason for GCDO being efficient is the use bounding constraints to estimate costs and jump over suboptimal total orders without explicitly calling sub-solvers. This can be seen from that the ratios $\zeta$ of GCDO and CDITO differ by two or three orders of magnitude, which means GCDO avoids a large number of unnecessary calls to the evaluation function and instead focuses on thoroughly exploring the solution space. 

\section{Conclusion}
\label{sec:conclusion}

In this paper, we presented GCDO, a generalized conflict-directed search algorithm that efficiently solves the optimal ordering problems with tightly coupled temporal and state constraints. GCDO adopts the branch-and-bound to search a special total order tree and generalizes the ordering conflicts in CDITO to bounding constraints, which can summarize both inconsistency and suboptimality. Thus, GCDO is able to efficiently prune the inconsistent or suboptimal total orders and thus avoids expensive and unnecessary calls to the sub-solvers. In our experiments on optimal temporal network configuration problems generated by a communication network simulator, we empirically demonstrate the efficiency of GCDO over CDITO and a MILP encoding.

\section*{Acknowledgments}
This work at Massachusetts Institute of Technology was supported by Kawasaki Heavy Industries, Ltd (KHI) under grant number 030118-00001. This article solely reflects the opinions and conclusions of its authors and not KHI or any other Kawasaki entity.

\small{\bibliography{bib}}

\begin{thebibliography}{18}
\providecommand{\natexlab}[1]{#1}
\providecommand{\url}[1]{\texttt{#1}}
\providecommand{\urlprefix}{URL }
\expandafter\ifx\csname urlstyle\endcsname\relax
  \providecommand{\doi}[1]{doi:\discretionary{}{}{}#1}\else
  \providecommand{\doi}{doi:\discretionary{}{}{}\begingroup
  \urlstyle{rm}\Url}\fi

\bibitem[{Balas and Xue(1996)}]{balas1996weighted}
Balas, E.; and Xue, J. 1996.
\newblock Weighted and unweighted maximum clique algorithms with upper bounds
  from fractional coloring.
\newblock \emph{Algorithmica} 15(5): 397--412.

\bibitem[{Bomze et~al.(1999)Bomze, Budinich, Pardalos, and
  Pelillo}]{bomze1999maximum}
Bomze, I.~M.; Budinich, M.; Pardalos, P.~M.; and Pelillo, M. 1999.
\newblock The maximum clique problem.
\newblock In \emph{Handbook of combinatorial optimization}, 1--74. Springer.

\bibitem[{Cai and Lin(2016)}]{cai2016fast}
Cai, S.; and Lin, J. 2016.
\newblock Fast Solving Maximum Weight Clique Problem in Massive Graphs.
\newblock In \emph{IJCAI}, 568--574.

\bibitem[{Chen et~al.(2018)Chen, Fang, Muise, Shrobe, Williams, and
  Yu}]{chen2018radmax}
Chen, J.; Fang, C.; Muise, C.; Shrobe, H.; Williams, B.~C.; and Yu, P. 2018.
\newblock RADMAX: Risk And Deadline Aware Planning for Maximum Utility.
\newblock In \emph{AAAI Workshop on Artificial Intelligence for Cyber Security
  (AICS'18)}.

\bibitem[{Chen et~al.(2019)Chen, Fang, Wang, Wang, and
  Williams}]{chen2019efficiently}
Chen, J.; Fang, C.; Wang, D.; Wang, A.; and Williams, B. 2019.
\newblock Efficiently Exploring Ordering Problems through Conflict-Directed
  Search.
\newblock In \emph{Proceedings of the International Conference on Automated
  Planning and Scheduling}, volume~29, 97--105.

\bibitem[{Dago and Verfaillie(1996)}]{dago1996nogood}
Dago, P.; and Verfaillie, G. 1996.
\newblock Nogood recording for valued constraint satisfaction problems.
\newblock In \emph{Proceedings Eighth IEEE International Conference on Tools
  with Artificial Intelligence}, 132--139. IEEE.

\bibitem[{De~Moura and Bj{\o}rner(2008)}]{de2008z3}
De~Moura, L.; and Bj{\o}rner, N. 2008.
\newblock Z3: An efficient SMT solver.
\newblock In \emph{International conference on Tools and Algorithms for the
  Construction and Analysis of Systems}, 337--340. Springer.

\bibitem[{Dechter, Meiri, and Pearl(1991)}]{dechter1991temporal}
Dechter, R.; Meiri, I.; and Pearl, J. 1991.
\newblock Temporal constraint networks.
\newblock \emph{Artificial intelligence} 49(1-3): 61--95.

\bibitem[{Gurobi~Optimization(2020)}]{gurobi2020gurobi}
Gurobi~Optimization, I. 2020.
\newblock Gurobi optimizer reference manual.
\newblock \emph{URL http://www. gurobi. com} .

\bibitem[{Lawler and Wood(1966)}]{lawler1966branch}
Lawler, E.~L.; and Wood, D.~E. 1966.
\newblock Branch-and-bound methods: A survey.
\newblock \emph{Operations research} 14(4): 699--719.

\bibitem[{Manne(1960)}]{manne1960job}
Manne, A.~S. 1960.
\newblock On the job-shop scheduling problem.
\newblock \emph{Operations Research} 8(2): 219--223.

\bibitem[{Ono and Nakano(2005)}]{ono2005constant}
Ono, A.; and Nakano, S.-i. 2005.
\newblock Constant time generation of linear extensions.
\newblock In \emph{FCT}, 445--453. Springer.

\bibitem[{Schiex et~al.(1995)Schiex, Fargier, Verfaillie
  et~al.}]{schiex1995valued}
Schiex, T.; Fargier, H.; Verfaillie, G.; et~al. 1995.
\newblock Valued constraint satisfaction problems: Hard and easy problems.
\newblock \emph{IJCAI (1)} 95: 631--639.

\bibitem[{Sebastiani and Trentin(2020)}]{sebastiani2020optimathsat}
Sebastiani, R.; and Trentin, P. 2020.
\newblock OptiMathSAT: A tool for optimization modulo theories.
\newblock \emph{Journal of Automated Reasoning} 64(3): 423--460.

\bibitem[{Timmons and Williams(2020)}]{timmons2020best}
Timmons, E.~M.; and Williams, B.~C. 2020.
\newblock Best-First Enumeration Based on Bounding Conflicts, and its
  Application to Large-scale Hybrid Estimation.
\newblock \emph{Journal of Artificial Intelligence Research} 67: 1--34.

\bibitem[{Wang(2015)}]{wang2015}
Wang, D. 2015.
\newblock \emph{A Factored Planner for the Temporal Coordination of Autonomous
  Systems}.
\newblock Ph.D. thesis, Massachusetts Institute of Technology.

\bibitem[{Wang and Williams(2015)}]{wang2015tburton}
Wang, D.; and Williams, B. 2015.
\newblock tBurton: A Divide and Conquer Temporal Planner.
\newblock In \emph{Twenty-Ninth AAAI Conference on Artificial Intelligence}.

\bibitem[{Yu and Williams(2013)}]{yu2013continuously}
Yu, P.; and Williams, B.~C. 2013.
\newblock Continuously relaxing over-constrained conditional temporal problems
  through generalized conflict learning and resolution.
\newblock In \emph{Twenty-Third International Joint Conference on Artificial
  Intelligence}.

\end{thebibliography}
\end{document}